%% file: arxiv.tex
\documentclass[sigconf, nonacm]{aamas} 

\usepackage{balance} 
\definecolor{lightcoral}{HTML}{F08080}
\definecolor{plum2}{HTML}{DDA0DD}
\definecolor{gold}{HTML}{FFD700}

\usepackage{colortbl}
\usepackage{balance} 
\usepackage{subcaption}
\usepackage[ruled,vlined, linesnumbered]{algorithm2e}
\usepackage{ragged2e}
\usepackage{enumitem}
\usepackage{multirow}
\usepackage{makecell} 

\DeclareMathOperator*{\argmax}{arg\,max}

\definecolor{Aquamarine}{RGB}{102,205,170}
\definecolor{Thistle}{RGB}{224,105,255}
\definecolor{Salmon}{RGB}{240,128,128}
\definecolor{Mygold}{RGB}{255,215,0}

\makeatletter
\def\@copyrightspace{\relax}
\makeatother

\setcopyright{none}

\acmSubmissionID{<<OpenReview submission id>>}

\title[AAMAS-2025 Formatting Instructions]{Adaptive Bi-Level Multi-Robot Task Allocation and Learning under Uncertainty with Temporal Logic Constraints}

\author{Xiaoshan Lin}
\affiliation{
  \institution{University of Minnesota, Twin Cities}
  \city{Minneapolis}
  \country{United States}}
\email{lin00668@umn.edu}

\author{Roberto Tron}
\affiliation{
  \institution{Boston University}
  \city{Boston}
  \country{United States}}
\email{tron@bu.edu}

\begin{abstract}
This work addresses the problem of multi-robot coordination under unknown robot transition models, ensuring that tasks specified by Time Window Temporal Logic are satisfied with user-defined probability thresholds. We present a bi-level framework that integrates (i) high-level task allocation, where tasks are assigned based on the robots' estimated task completion probabilities and expected rewards, and (ii) low-level distributed policy learning and execution, where robots independently optimize auxiliary rewards while fulfilling their assigned tasks. To handle uncertainty in robot dynamics, our approach leverages real-time task execution data to iteratively refine expected task completion probabilities and rewards, enabling adaptive task allocation without explicit robot transition models. We theoretically validate the proposed algorithm, demonstrating that the task assignments meet the desired probability thresholds with high confidence. Finally, we demonstrate the effectiveness of our framework through comprehensive simulations.
\end{abstract}


\newcommand{\BibTeX}{\rm B\kern-.05em{\sc i\kern-.025em b}\kern-.08em\TeX}

\begin{document}
\include{commands.tex}


\pagestyle{fancy}
\fancyhead{}


\maketitle 


\section{Introduction}
Efficient and high-quality task allocation plays a critical role in cooperative multi-robot applications, such as on-demand ridesharing and delivery \cite{hyland2018dynamic}, assembly lines \cite{johannsmeier2016hierarchical}, and warehouse logistics \cite{yan2012multi}. Recently, multi-robot task allocation with temporal constraints has attracted growing attention due to its significance in time-sensitive applications. For instance, \cite{suslova2020multi} formulates a distributed constraint optimization problem to allocate tasks with time windows and ordering requirements to heterogeneous robots. 
\cite{choudhury2022dynamic} considers uncertainty in completing time-windowed tasks and combines high-level task allocation with low-level scheduling to minimize task incompletion. \cite{park2023risk} addresses tasks with precedence constraints, proposing an algorithm that ensures the probability of task failure remains below a user-specified threshold under uncertainties in robot traits. Research in this area typically considers tasks that can be executed independently. 

Multi-robot systems often involve more complicated scenarios where tasks are logically and temporally correlated. For instance, in warehouse logistics, robots need to coordinate to retrieve items from multiple storage locations based on orders, inventory, and replacement options, then assemble and deliver them within a set time frame, with each step depending on the successful completion of the previous one. Temporal logic (TL) provides a rigorous framework for specifying such temporal order and dependencies between tasks.

Different methods have been proposed to coordinate robots to satisfy a global temporal logic formula, including graph search on product automata \cite{schillinger2018simultaneous}, sampling-based methods \cite{luo2021abstraction, kantaros2020stylus}, and optimization-based approaches \cite{karaman2011linear, wolff2014optimization, buyukkocak2021planning}. 
Other works assign local temporal logic tasks to individual robots, employing methods such as receding horizon planning \cite{tumova2016multi}, path-finding on product graph \cite{gujarathi2022mt}, and integer programming-based planning \cite{sun2022multi}. 
Nevertheless, these studies do not address uncertainty, which is crucial for time-sensitive tasks. For instance, environmental factors like traffic jams can disrupt task execution within the required time window.

 Some prior studies \cite{cai2021modular, jiang2021temporal, hasanbeig2020cautious, aksaray2021probabilistically} employ reinforcement learning (RL) to maximize the probability of satisfying temporal logic constraints under unknown transition dynamics, primarily focusing on single-robot systems. Multi-agent reinforcement learning (MARL) has been explored for joint policy learning with temporal logic objectives \cite{hammond2021multi, wang2023multi, zhang2022distributed}. Additionally, semi-decentralized methods \cite{sun2020automata}, deep RL \cite{muniraj2018enforcing}, and model-based RL \cite{liu2023catlnet} have been explored to improve learning efficiency. However, these MARL approaches either do not explicitly address system model uncertainty or assume full system knowledge, limiting their applicability in uncertain environments. 

Overall, there are research gaps in existing multi-robot task allocation, planning, and learning approaches, as they either assume the tasks are independent, neglect the uncertainty in system models, or unrealistically presume full knowledge of uncertainties. Additionally, prior studies often aim to maximize robustness degree \cite{muniraj2018enforcing, wang2023multi}, minimize incomplete tasks \cite{choudhury2022dynamic}, or optimize some cost functions \cite{kantaros2020stylus, luo2021abstraction}. While these approaches are effective, they often lack quantifiable guarantees on task satisfaction. 

Motivated by these gaps, this paper addresses multi-robot task allocation to ensure a specified probability of satisfying Time Window Temporal Logic tasks under unknown transition dynamics. Additionally, a secondary objective is to maximize auxiliary reward functions for each robot based on user preferences. Fig.~\ref{fig:example} illustrates an example where robots must transport materials within specific time windows with high probability while also performing auxiliary tasks such as identifying and monitoring traffic-prone intersections or returning to stations for future assignments.

Our proposed formulation, which incorporates primary tasks and auxiliary objectives, has broader applicability in various scenarios, such as: \textbf{(i) Search-and-rescue:} Robots deliver supplies within time-critical windows while minimizing energy use or maximizing coverage to identify other areas needing help. \textbf{(ii) Environmental monitoring:} Drones monitor areas during specific intervals, such as wildlife peaks, while collecting additional data or conserving power for longer missions.  

In summary, this work makes the following contributions:
\begin{itemize}[noitemsep,topsep=5pt,parsep=5pt,partopsep=5pt, leftmargin=0.5cm]
\item We propose a framework for multi-robot task allocation under unknown robot transition probabilities, integrating high-level task allocation with low-level policy learning to ensure probabilistic satisfaction of time-window temporal logic tasks.

\item We incorporate auxiliary reward functions into high-level task allocation as a secondary objective, allowing the optimization of auxiliary rewards to impose user preferences without compromising the primary temporal logic tasks.

\item We use data-driven methods to estimate expected rewards and task satisfaction probabilities of robots while following their low-level policies, which enable adaptive high-level task allocation using these refined reward functions and satisfaction probabilities without requiring explicit robot transition models.
\end{itemize}

The most relevant works in the literature are \cite{schillinger2019hierarchical, lin2024probabilistic, aksaray2021probabilistically, choudhury2022dynamic, sun2020automata, park2023risk}. \cite{choudhury2022dynamic} employ a similar bi-level framework but assumes independent tasks with single time windows and known task completion probabilities,  whereas we consider complex spatio-temporal dependencies and unknown probabilities. 
In \cite{park2023risk}, the Sequential Probability Ratio Test (SPRT) is used to verify whether the probability of task failure remains below a user-defined threshold. This approach iteratively recomputes task assignments until passing the SPRT, whereas we formulate an optimization problem whose solution guarantees probabilistic task satisfaction. Both \cite{aksaray2021probabilistically} and \cite{lin2024probabilistic} consider similar problem formulations, aiming to satisfy a temporal logic task with a desired probability.
While these works focus on a single agent with a single task, we address multi-robot systems with multiple tasks.  \cite{sun2020automata} considers a similar communication model in which each agent interacts with a centralized coordinator.
However, their approach involves the coordinator in lower-level policy learning, whereas the coordinator is solely used for high-level task allocation in our framework. \cite{schillinger2019hierarchical} uses RL to estimate uncertain task execution durations and incorporates the learned durations into an auction-based allocation policy to minimize execution time. Similarly, our approach uses RL to estimate the reward functions and improve task allocation. Our work specifically focuses on satisfying tasks within exact time windows with desired probabilities, which their work does not address.


\section{Preliminaries}
 Let $\Sigma$ be a finite set. We denote the power set of $\Sigma$ by $2^\Sigma$. A finite or infinite sequence of elements from $\Sigma$ is called a word over $\Sigma$. In this context, $\Sigma$ is also called an alphabet. 
 Let $k, i, j \in \mathbb{Z}_{\geq0}$ with $i \leq j$. The $k$-th element of a word $\boldsymbol{\sigma}$ is denoted by $\boldsymbol{\sigma}_k$, and the subword $\boldsymbol{\sigma}_i$ , ... , $\boldsymbol{\sigma}_j$ is denoted by $\boldsymbol{\sigma}_{i, j}$. We denote the set of atomic propositions by $AP$.

 \begin{figure}[t]
  \centering
   \includegraphics[width=0.85\linewidth]{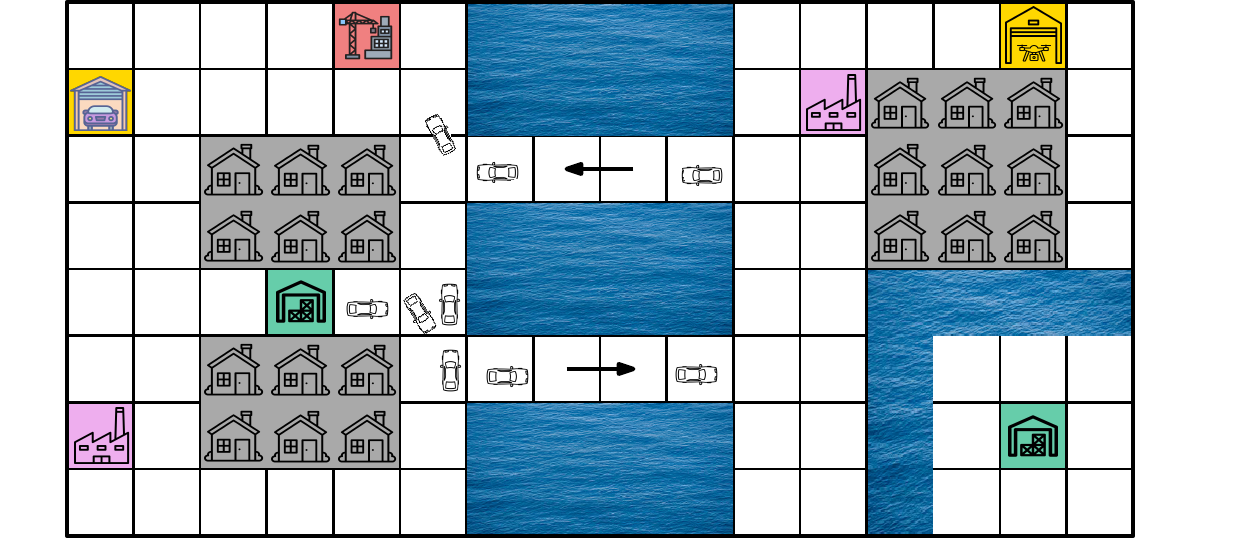}
  \caption{
  Motivating example of our proposed framework. Primary tasks:
  repeatedly transporting resources from \textcolor{Aquamarine}{warehouses} to \textcolor{Thistle}{processing stations} and then to \textcolor{Salmon}{operation site} within specified time windows, which must be satisfied with a high probability to prevent resource accumulation. Auxiliary tasks
  : monitoring traffic congestion to improve future routing and delivery times, or returning to \textcolor{Mygold}{stations} after completing the primary tasks.}
  \label{fig:example}
  \Description{Logo of AAMAS 2025 -- The 24th International Conference on Autonomous Agents and Multiagent Systems.}
\end{figure}

\subsection{Time Window Temporal Logic}
Time Window Temporal Logic (TWTL) \cite{twtl} is a language for expressing time-bounded specifications. 
A TWTL formula is defined over a set of atomic propositions $AP$ as follows:
    \begin{equation*}
        \phi ::= H^ds\,\,|\,\, H^d\neg s \,\,|\,\, \neg\phi_1 \,\,|\,\, \phi_1 \land \phi_2 \,\,|\,\, \phi_1 \lor \phi_2 \,\,|\,\, \phi_1 \cdot \phi_2 \,\,|\,\, [\phi_1]^{[a,b]}.
    \end{equation*}
Here, $s$ represents either the constant ``true" or an atomic proposition in $AP$;  $\phi_1$ and $\phi_2$ are TWTL formulas. The hold operator $H^ds$, with $d \in \mathbb{Z}_{\geq0}$, specifies that $s \in AP$ should hold for d time units. The negation operator $\neg\phi_1$ specifies ``do not satisfy the formula''. The conjunction operator $\phi_1 \land \phi_2$ and disjunction operator $\phi_1 \lor \phi_2$ specify ``satisfy both formulas'', and ``satisfy at least one formula'', respectively. The concatenation operator $\phi_1 \cdot \phi_2$ specifies that $\phi_1$ must be satisfied first, and $\phi_2$ must be satisfied immediately after. The within operator $[]^{[a,b]}$, where $a, b \in \mathbb{Z}_{\geq0}$ and $a \leq b$, restricts the satisfaction of $\phi$ to the time window $[a, b]$. The \textbf{\textit{time bound}} of a TWTL formula $\phi$, denoted as $\|(\phi)\|$, represents the maximum time allowed to satisfy $\phi$. For a formal definition of the TWTL semantics and the time bound, we refer readers to \cite{twtl}.

\begin{definition} (Deterministic Finite-State Automaton)
A deterministic finite-state automaton (DFA) is a tuple $\mathcal{A} = (Q, q_0, \Sigma, \delta, F),$ where $Q$ is a finite set of states, $q_0$ is the initial state, $\Sigma=2^{AP}$ is the input alphabet, $\delta : Q \times \Sigma \rightarrow Q$ is the transition function, and $F$ is the set of accepting states.
\end{definition}

 A finite input word $\boldsymbol{\sigma}= \sigma_0, \sigma_1, \dots, \sigma_T$ over the alphabet $2^{AP}$ generates a trajectory $\boldsymbol{q} = q_0, q_1, \dots, q_{T+1}$ on the DFA, where $q_0$ is the initial state of the DFA and $q_{k+1} = \delta(q_k, \sigma_k)$ for all $k \geq 0$. A finite input word $\boldsymbol{\sigma }$ over $\Sigma$ is accepted by a DFA if the corresponding trajectory $\boldsymbol{q}$ ends in an accepting state of the DFA. 

A TWTL formula $\phi$ can be translated into a DFA that either accepts or rejects an input word \cite{twtl}. An input word $\boldsymbol{\sigma }$ is said to satisfy the corresponding TWTL formula $\phi$ if it is accepted by the DFA, denoted as $\boldsymbol{\sigma} \models \phi$. 
For example, the TWTL formula $\phi=[H^1 P]^{[1,2]} \cdot [H^0 D]^{[0,2]}$, specifying a pickup ($P$) and hold for one time step within time window [1, 2] followed by a delivery ($D$) within two time steps, is translated into a DFA shown in Fig.~\ref{fig:normal_dfa}. An input word $\boldsymbol{\sigma}= \{P\}, \{P\}, \{P\}, \{D\}$ satisfies the TWTL formula, as it generates a trajectory $\boldsymbol{q}= q_0, q_1, q_2, q_3, q_6$ that ends in the accepting state $q_6$. This TWTL formula can also be translated into a temporally relaxed DFA (see Fig.~\ref{fig:relaxed_dfa}), which is more compact and facilitates more efficient computation and construction.

\begin{figure}[t]
    \centering
    \begin{subfigure}[t]{0.29\linewidth}
    \includegraphics[width=\linewidth]{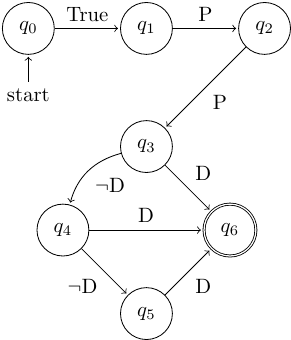}
        \caption{DFA}
        \label{fig:normal_dfa}
    \end{subfigure}
    \hspace{9pt}
    \begin{subfigure}[t]{0.22\linewidth}
    \includegraphics[width=\linewidth]{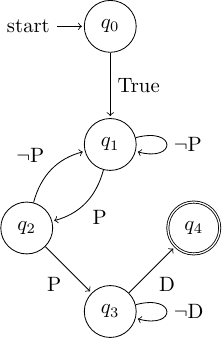}
        \caption{Relaxed DFA}
        \label{fig:relaxed_dfa}
    \end{subfigure}
    \hspace{9pt}
    \begin{subfigure}[t]{0.29\linewidth}
    \includegraphics[width=\linewidth]{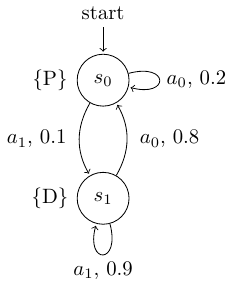}
        \caption{Labeled MDP}
        \label{fig:mdp}
    \end{subfigure}
    \caption{(a) Normal DFA for TWTL formula $\phi=[H^1 P]^{[1,2]} \cdot [H^0 D]^{[0,2]}$. (b) DFA for the temporally relaxed TWTL formula $\phi(\boldsymbol{\tau})=[H^1 P]^{[1,2+\tau_1]} \cdot [H^0 D]^{[0,2+\tau_2]}$. (c) Example of a labeled MDP, where $S = \{s_0, s_1\}$, $A = \{a_0, a_1\}$, $AP = \{P, D\}$, $l(s_0) = \{P\}$ and $l(s_1) = \{D\}$.}
\label{fig:graph}
\end{figure}
\subsection{Markov Decision Process}

\begin{definition}
(Labeled MDP) A labeled MDP is a tuple $\mathcal{M}=(S,A,\Delta,R,l)$, where $S$ represents the state space, and $A$ denotes the set of actions. The probabilistic transition function is given by $\Delta: S \times A \times S \rightarrow [0,1]$, and the reward function is defined as $R: S \times A \rightarrow \mathbb{R}$. Additionally, $l: S \rightarrow 2^{AP}$ is a labeling function that maps each state to a set of atomic propositions.
\end{definition}

\begin{definition}[Policy]
A policy over the MDP is a mapping $\pi: S \rightarrow P(A)$, where $P(A)$ represents the set of probability distributions over the actions set $A$. A policy is stationary if $\pi(\cdot | s) \in P(A)$ does not change over time, and it is deterministic if $\pi(\cdot | s)$ always selects the same action for a given state.
\end{definition}

MDPs with labeling functions are widely used in RL with temporal logic \cite{fu2014probably, neider2021advice, cai2021reinforcement} to allow verification of temporal logic constraints. Given a finite MDP trajectory $\boldsymbol{s} = s_0, s_1, \dots, s_T$, the labeling function generates an output word $\boldsymbol{o} = o_0, o_1, \dots, o_T$ where $o_i = l(s_i)$.  A subword of the output, $o_{t_1}, o_{t_1+1}, \dots, o_{t_2}$, is denoted as $\boldsymbol{o}_{t_1, t_2}$. Consider the labeled MDP example shown in Fig.~\ref{fig:mdp}. An MDP trajectory $\boldsymbol{s} = s_0, s_0, s_0, s_1$ generates an output word $\boldsymbol{o} = \{P\}, \{P\}, \{P\}, \{D\}$, which serves as an input word to the DFA to verify whether the robot satisfies the TWTL formula.

\subsection{System Setup}
We consider a multi-robot system with a set of $N$ robots, denoted as $\{r_i\}_{i \in \{1, \dots, N\}}$, and we refer to this set as $\{N\}$ for simplicity. Similarly, we define a set of $K$ tasks, denoted by $\{t_k\}_{k \in \{1, \dots, K\}}$, and refer to it as $\{K\}$. Each task $t_k \in \{K\}$ is defined by a tuple $(\phi_k, P_k)$, where $\phi_k$ is a TWTL formula and $P_k$ is the desired probability threshold for satisfying $\phi_k$. The time bound for these tasks is $T = \|\phi_k\|$ time steps, and they must be fulfilled every $T$ time steps over an infinite horizon or a long period. Each repetition of $T$ time steps is called an \textbf{\textit{episode}}. We introduce a virtual task \( t_{K+1} = (\phi_{K+1}, P_{K+1}) \), where \( \phi_{K+1} \) is a null task and \( P_{K+1}=0 \), to augment the task set. A robot assigned to task \( t_{K+1} \) is considered unassigned to any TWTL task, allowing it to focus on its auxiliary objective. We denote this augmented task set as \( \{K+1\} \).

We assume that the robots in $\{N\}$ have decoupled dynamics and reward functions, and operate independently without interactions with others. 
Each robot $r_i \in \{N\}$ is modeled as a decoupled MDP, denoted by $\mathcal{M}^i = (S^i, A^i, \Delta^i, R^i, l^i)$. The collection of all robot MDPs $\{\mathcal{M}^i\}_{i \in \{N\}}$ is referred to as $\{\mathcal{M}\}$. The state transition probabilities $\Delta^i$ are unknown and account for uncertainties such as movement delays due to road conditions or potential robot failures. The reward function \( R^i \) encodes the robot's individual auxiliary objectives.

The output word of robot $r_i$ over a time interval $[t_1, t_2]$ is denoted as $\boldsymbol{o}^i_{t_1, t_2}$. \textbf{\textit{Team output word}}, collection of output words from all robots $\{\boldsymbol{o}^i_{t_1, t_2}\}_{i\in\{N\}}$, is given by $\boldsymbol{o}_{t_1, t_2}$. Given a TWTL formula $\phi$, we say $\{\boldsymbol{o}_{t_1, t_2}\} \models \phi$ if and only if there exists at least one robot $r_i$ such that $\boldsymbol{o}^i_{t_1, t_2} \models \phi$. The probability that $\{\mathbf{o}_{t_1, t_2}\}$ satisfies a TWTL formula $\phi$ is denoted as $\Pr\big(\{\mathbf{o}_{t_1, t_2}\} \models \phi  \big)$. 

Guaranteeing probabilistic constraint satisfaction is unachievable without prior knowledge of the MDP transition probabilities. In this work, we assume that although the exact transition probabilities are unknown, the agent has partial knowledge: when taking action $a$ at state $s$, the agent knows (i) which states have a non-zero probability of being reached and (ii) which states have a sufficiently high probability of being the next state. This partial knowledge is represented by $\epsilon^i$, an upper bound on transition uncertainty, which is crucial for computing a lower bound on the probability of task satisfaction. We formalize this assumption as follows:

\begin{assumption}\label{assumption} 
Given an MDP $\mathcal{M}^i$ for an arbitrary robot $r_i$, the exact transition probability $\Delta^i$ is unknown. However, a known parameter $\epsilon^i \in [0,1)$ exists such that, for each state $s$ and action $a$, (i) the set of states $s'$ for which $\Delta^i(s, a, s') > 0$ and (ii) the set of states $s''$ for which $\Delta^i(s, a, s'') \geq 1 - \epsilon^i$ are non-empty and known. 
\end{assumption}


\begin{example}\label{example:transition}
    Consider a probabilistic transition model shown in Fig.~\ref{fig:transitions}, where the action set is given by $A = \{N,NE,E,SE,S,$ $SW,W, NW, Stay\}$. Under the `Stay' action, robots remain stationary, while all other actions result in an intended transition (black arrows) with probability $1-\epsilon$ and unintended transitions (red arrows) with a probability of $\epsilon$. The exact values of $\epsilon$ are unknown to robots. However, each robot $r_i$ is provided with an estimate $\epsilon^i$, which overestimates the actual $\epsilon$ and reflects the robots' limited knowledge of their transition models and environmental disturbances, as stated  Assumption~\ref{assumption}.
\end{example}


\section{Problem Formulation}
Given a multi-robot system $\{N\}$ and a set of tasks $\{K+1\}$, the objective is to find an optimal policy for each MDP in $\{\mathcal{M}\}$ to ensure that the team output word $\{\boldsymbol{o}_{jT, \,jT+T}\}$ satisfies each TWTL formula $\phi_k$ with at least probability $P_k$ while maximizing the sum of rewards. With unknown MDP transitions, this can be formulated as a constrained RL problem, with the probabilistic satisfaction of TWTL tasks encoded as constraints, and reward maximization as the objective.
This problem presents multiple challenges, including coordinating robots under uncertain dynamics, ensuring probabilistic satisfaction of TWTL tasks, and maximizing individual rewards. 

We propose a bi-level solution to address these challenges hierarchically. At the high level, we solve for multi-robot task allocation to ensure desired probabilistic satisfaction of the TWTL tasks. Specifically, we consider task allocation $\Pi: \{N\} \times \{K + 1\} \rightarrow [0, 1]$, where $\Pi(i, k)$ represents the probability of robot $r_i$ being assigned to task $t_k$. At the low level, each robot independently finds its own policy $\pi^{\Pi}_{i}$, to satisfy its assigned task while simultaneously maximizing its individual reward. We formally define the problem as below.

\begin{problem}\label{problem}     
    Given a set of robots $\{N\}$, a set of TWTL tasks $\{K+1\}$, a set of MDP $\{\mathcal{M}\}$ subject to Assumption 1, and a discount factor $\gamma$, determine the optimal high-level task allocation and low-level policies 
        \begin{equation}
           \Pi^*, \pi_1^*, \dots, \pi_N^*  = \argmax_{\Pi, \pi_1^\Pi, \dots, \pi_N^\Pi} \mathbb{E}_{\left\{a_t^i\sim \pi_i^\Pi\right\}} \left[ \sum_{t=0}^{\infty} \gamma^t \sum_{i=1}^N R^i(s_t^i, a_{t}^i)  \right],
        \end{equation}
        such that, for every episode j,
        \begin{equation}\label{problem_constraint}
        \begin{aligned} 
            \Pr\big(\{\mathbf{o}_{jT,jT + T}\} \models \phi_k \big) \color{black} & \geq P_k, \quad \forall j\geq0, (\phi_k, P_k) \in \{K+1\} 
        \end{aligned}
        \end{equation}
        \noindent where $\{\mathbf{o}_{jT,jT + T}\}$ is the team output word in episode j. 
\end{problem}

\begin{remark}
    The reward functions can incorporate different auxiliary objectives into task allocations based on system requirements: \textbf{(i) Environment-driven vs. user-defined:} robots can learn rewards from environmental interactions (e.g., identifying high-traffic areas) or follow predefined incentives (e.g., encouraging timely returns to charging stations). \textbf{(ii) Robot-specific objectives:} reward functions can be designed to reflect the distinct roles of different robots (e.g. aerial robots prioritize large area coverage, while mobile robots remain in critical locations to improve future task efficiency). \textbf{(iii) Capability-aware rewards:} rewards encode robot capabilities, such as higher rewards for drones with superior sensors when monitoring key areas.
 
\end{remark}

\section{Proposed Solution}

 Our approach consists of a high-level coordinator for task allocation and a low-level task execution mechanism. The high-level coordinator considers each robot’s probability of satisfying TWTL tasks from its current state and the expected rewards associated with specific tasks. It computes task assignments that maximize the total expected reward while ensuring that each task is completed with the desired probability threshold. At the low level, each robot maintains \( K+1 \) policies: \( K \) stationary policies, one for each TWTL task, and an additional policy learned over time to maximize individual rewards when not assigned to a TWTL task. Upon receiving an assignment from the high-level coordinator, each robot follows the corresponding policy for its designated task.

\subsection{High-level Multi-robot Task Assignment}

Before each episode, each robot provides the coordinator with its $P^{\epsilon}_{i, k}(p_i)$, the expected probability of satisfying TWTL formula $\phi_k$, and $V_{i, k}(p_i)$, the expected reward under its current policy. Here, \( p_i \) denotes the current state of the robot. For simplicity, we refer to these quantities as \( P^{\epsilon}_{i, k} \) and \( V_{i, k} \) in the following sections.

For now, we assume that $P^{\epsilon}_{i, k}$ and $V_{i, k}$ are known to the robots. The process of obtaining these values will be discussed in later sections. Given a set of tasks $\{K+1\}$, the high-level task assignment for robot $r_i \in \{N\}$ is represented by the probability distribution $\{P_{i,k}\}_{k=1, \dots, K + 1}$, where $\sum_{k=1}^{K + 1} P_{i,k} = 1$. For $k = 1, 2, \dots, K$, $P_{i,k}$ denotes the probability of robot $r_i$ being assigned to TWTL task $t_k$, while $P_{i, K+1}$ denotes the probability that robot $r_i$ prioritizes maximizing its reward function instead of executing a TWTL task. Following the multi-robot task allocation taxonomy proposed by \cite{gerkey2004formal}, our task allocation problem falls under ST-SR-IA, i.e. single-task robots (ST) execute single-robot tasks (SR), with instantaneous assignments (IA). Note that although the stochastic task allocation allows multiple robots to be assigned to the same task, we assume they execute tasks in parallel rather than collaboratively. Therefore, tasks are classified as single-robot rather than multi-robot. The high-level task allocation is determined by solving the following optimization problem.

 \begin{subequations} \footnotesize\label{eqn:optimization}
\begin{align}
    \arg \max_{P_{i,k}} \quad & \sum_{i=1}^{N} \sum_{k=1}^{K+1} P_{i,k} \cdot V_{i, k} \label{eqn:obj}\\
   \text{subject to} 
    \quad & 1 - \prod_{i=1}^{N} (1 - P_{i,k} \cdot P^{\epsilon}_{i, k})\geq P_k \quad \forall k = 1, 2, \dots, K \label{eqn:constraint1}\\
    & \sum_{k=1}^{K+1} P_{i,k} = 1 \quad \forall i = 1, 2, \dots, N \label{eqn:constraint2}\\
    & P_{i,k} \geq 0 \quad \forall i = 1, 2, \dots, N, k  = 1, 2, \dots, K+1 \label{eqn:constraint3}
\end{align}
\end{subequations}

In problem \eqref{eqn:optimization}, the objective function (\ref{eqn:obj}) aims to maximize the sum of the expected rewards across all agents. Constraint (\ref{eqn:constraint1}) ensures that the probability of at least one robot satisfying task $t_k$ is not less than its desired probability $P_k$. Constraints (\ref{eqn:constraint2}) and (\ref{eqn:constraint3}) ensure that the task assignment $\{P_{i, k}\}$ define a valid probability distribution over tasks for each robot. Although the assignments are stochastic, each robot selects and executes only one task at a time, sampled from the task assignment. Unlike deterministic solutions commonly used in multi-agent task allocation, our approach allows robots to optimize auxiliary rewards in addition to satisfying TWTL tasks. Notably, deterministic task allocation is a special case of our framework, occurring when $P_{i, k}$ are either 1 or 0.

 The probability \( P^{\epsilon}_{i,k} \) in \eqref{eqn:constraint1} 
 is unknown due to the lack of knowledge about the transition probabilities in the MDPs. One solution is to substitute \( P^{\epsilon}_{i,k} \) with its lower bound. The key question is whether solving problem \eqref{eqn:optimization} using lower bounds of \( P^{\epsilon}_{i,k} \) still ensures satisfaction of the original constraints. To answer this, we present the following proposition.

\begin{proposition}\label{proposition}
Let $\lfloor{P}^{\epsilon}_{i,k}\rfloor$ be an arbitrary lower bound for ${P}^{\epsilon}_{i,k}$, that is, $ 0 \leq \lfloor{P}^{\epsilon}_{i,k}\rfloor\leq {P}^{\epsilon}_{i,k} \leq 1$. If a set of probabilities $\{P_{i,k}\}$ satisfies 
\begin{equation}\footnotesize
    1 - \prod_{i=1}^{N} (1 - P_{i,k} \cdot \lfloor{P}^{\epsilon}_{i,k}\rfloor)\geq P_k \quad \forall k = 1, 2, \dots, K, 
\end{equation}
then $\{P_{i,k}\}$ also satisfies constraint \eqref{eqn:constraint1}.
\end{proposition}

\begin{proof}
Since ${P}^{\epsilon}_{i,k} \geq \lfloor{P}^{\epsilon}_{i,k}\rfloor$, it follows that:
\begin{equation}\footnotesize
    \prod_{i=1}^{N} (1 - P_{i,k} \cdot {P}^{\epsilon}_{i,k}) \leq \prod_{i=1}^{N} (1 - P_{i,k} \cdot \lfloor{P}^{\epsilon}_{i,k}\rfloor).
\end{equation}

\begin{equation}\footnotesize
    \implies 1 - \prod_{i=1}^{N} (1 - P_{i,k} \cdot {P}^{\epsilon}_{i,k}) \geq 1 - \prod_{i=1}^{N} (1 - P_{i,k} \cdot \lfloor{P}^{\epsilon}_{i,k}\rfloor).
\end{equation}

\end{proof}

This proposition guarantees that replacing \( P^{\epsilon}_{i,k} \) with its lower bound in \eqref{eqn:optimization} preserves the original task satisfaction constraints \eqref{eqn:constraint1}.

\subsection{Low-level Single-Agent Task Execution}
Upon receiving its task assignment \( \{P_{i,k}\}_{k=1,2,\ldots,K+1} \) from the high-level coordinator, robot $r_i$ samples from this probability distribution to determine which task it will execute in the upcoming episode.
The robot's policy during the episode depends on whether the selected task $t_k$ is a TWTL task ($1 \leq k \leq K$), or the auxiliary task focused on maximizing the robot’s individual reward  ($k = K+1$).

\subsubsection{Single-Agent Policy for TWTL Tasks ($1 \leq k \leq K$)} In this case, the robot follows a stationary policy that maximizes the probability of satisfying the TWTL task \( t_k \). We discuss the methodology for deriving this policy in the following section.

\begin{definition}
    \label{def:aug-pa}
    (Product MDP) Given a labeled MDP $\mathcal{M} = (S,A,\Delta,$ $R, l)$ and an DFA $\mathcal{A} = (Q, q_0, \Sigma, \delta,$ $F)$, a product MDP is a tuple $\mathcal{P} =\mathcal{M} \times \mathcal{A} = (S^{\otimes},S^{\otimes}_{0},A^{\otimes},\Delta^{\otimes}, R^{\otimes}, F^{\otimes})$, where $S^{\otimes} = S \times Q$ is a finite set of states; $S^{\otimes}_{0}=\{(s,\delta(q_0, l(s)))\, | \,\forall s \in S\}$ is the set of initial states; $A^{\otimes}=A$ is the set of actions; $\Delta^{\otimes}: S^{\otimes} \times A^{\otimes} \times S^{\otimes} \rightarrow [0,1]$ is the probabilistic transition relation such that for any two states, $p=(s,q) \in S^{\otimes}$ and $p'=(s^{\prime},q^{\prime}) \in S^{\otimes}$, and any action $a \in A^{\otimes}$, $\Delta^{\otimes}(p,a,p')=\Delta(s,a,s^\prime)$ if $\delta(q,l(s))=q^\prime$; $R^{\otimes}:S^{\otimes} \rightarrow \mathbb{R}$ is the reward function such that $R^{\otimes}(p) = R(s)$ for $p=(s,q) \in S^{\otimes}$; $F^{\otimes} = (S \times F) \subseteq S^{\otimes}$ is the set of accepting states.
\end{definition}  

    To maximize the probability of reaching $F^{\otimes}$ from any state of the product MDP, we can select the action that is most likely to minimize the expected distance to any accepting states. Inspired by \cite{lin2024probabilistic, aksaray2021probabilistically}, we define $\epsilon$-stochastic transitions and distance-to-$F^{\otimes}$.

\begin{definition}[$\epsilon$-Stochastic Transitions]
For a given product MDP and some $\epsilon \in [0,1)$, any transition $(p,a,p')$ such that $\Delta^{\otimes}(p,a,p')$ $\geq 1-\epsilon$ is defined as an $\epsilon$-stochastic transition.   
\end{definition}

This definition corresponds to Assumption \ref{assumption}, where transitions with sufficiently high probabilities (\(\geq 1 - \epsilon\)) are labeled as \(\epsilon\)-stochastic transitions. While this is a general definition, for each robot \( r_i \), \(\epsilon\) should be replaced with the robot-specific parameter \(\epsilon^i\).

\begin{definition}
    [Distance-To-$F^{\otimes}$]
    Given a product MDP, the distance from any state $p \in S^{\otimes}$ to the set of accepting states $F^{\otimes}$ is 
    \begin{equation}\label{dist-to-F}
        D^\epsilon(p) = \min\limits_{p' \in F^{\otimes}} dist^\epsilon(p,p'),
    \end{equation} 
    where $dist^\epsilon(p,p')$ represents the minimum number of $\epsilon$-stochastic transitions to move from $p$ to another state $p'$. 
\end{definition}

$D^\epsilon(p)$ represents the minimum number of $\epsilon$-stochastic transitions needed to move from state $p$ to the set of accepting states.


\begin{definition}
[$\pi^{\epsilon}$ Policy] Given a product MDP and $\epsilon \in [0,1)$, $\pi^{\epsilon} : S^{\otimes} \rightarrow A$, is a stationary policy over the product MDP such that
\begin{equation}\label{eq:pigo}
   \pi^{\epsilon}(p) = \arg \min_{a\in A} D^\epsilon_{min}(p,a),
\end{equation}
\text{where} $D^\epsilon_{min}(p,a) = \min \limits_{p':\Delta^{\otimes}(p,\,a,\,p')\geq 1-\epsilon} D^{\epsilon}(p')$, i.e., the minimum distance-to-$F_\mathcal{P}$ among the states reachable from $p$ under action $a$ with probability of at least $1-\epsilon$.
\end{definition}

We summarize the procedure for synthesizing the policy \( \pi^{\epsilon} \) in Alg.~\ref{alg:offline}. The inputs are the set of tasks and an MDP $\mathcal{M}^i$ that adheres to Assumption \ref{assumption} with parameter $\epsilon^i$. First, a DFA is generated for each TWTL formula (line 2). Then, for robot $r_i$ and its corresponding MDP, a product MDP is constructed (line 3). The product MDP is used to calculate the Distance-To-$F^{\otimes}$ for all states using $\epsilon^i$-stochastic transitions (line 4). Finally, the $\pi^{\epsilon}_{i, k}$ policy is computed for this specific product MDP by selecting the action that minimizes the Distance-To-$F^{\otimes}$ for each state (lines 5-6). 

\begin{algorithm}[htb!]
\small
\SetKwInOut{Input}{Input}
\SetKwInOut{Output}{Output}
\Input{\justifying{A set of tasks $\{K+1\}$; \,a MDP $\mathcal{M}^i$}}
\Output{Stationary policies \{$\pi_{i, k}^{\epsilon}\}$; product MDPs $\{\mathcal{P}_{i, k}\}$}
\caption{Off-line computation of $\pi^{\epsilon}$ policy}\label{alg:offline}
\DontPrintSemicolon
\For{$t_k = (\phi_k, P_k) \in \{K+1\}, \, k \leq K$}{
  $\mathcal{A}_k \leftarrow$ Create DFA of TWTL formula $\phi_k$; \\
      $\mathcal{P}_{i, k} = \mathcal{M}^i \times \mathcal{A}_k \leftarrow$ Create product MDP;\\
      Calculate Distance-To-$F^{\otimes}$ for all states in $\mathcal{P}_{i, k}$;\\
      \For{$p \in$  all states of $\mathcal{P}_{i, k}$}{
         $\pi_{i, k}^{\epsilon}(p) \leftarrow (\ref{eq:pigo})$
     }
  
}
\end{algorithm} 

Recall that we want to replace \( P^{\epsilon}_{i, k} \) with its lower bound in problem \eqref{eqn:optimization}. Now we will discuss methods for obtaining its lower bounds.  By definition, \( P^{\epsilon}_{i, k} \) is the probability that robot \( r_i \),  under stationary policy \( \pi^{\epsilon}_{i, k} \), satisfies the TWTL formula \( \phi_k \). Equivalently, it is the probability of reaching the accepting states of the product MDP from an initial state under \( \pi^{\epsilon}_{i, k} \). We adopt the method proposed in \cite{lin2024probabilistic}, which finds the lower bound of the probability of reaching accepting states within a time window, given the upper bounds of the transition uncertainties. We denote this lower bound as \( \underline{P}^{\epsilon}_{i,k} \) and refer to it as the \textbf{\textit{static lower bound}}. We define \( \{\underline{P}^{\epsilon}_{i,k} \}\) as the set of static lower bounds. Alternatively, another approach is to use the partial information from Assumption \ref{assumption} to estimate the upper and lower bounds of the transition probabilities in the MDP. These bounds can then be used with probabilistic model checkers such as PRISM \cite{KNP11} or the optimization-based method proposed in \cite{lin2023reinforcement} to compute the lower bounds.

While the static lower bound can be computed offline using partial information \( \epsilon^i \) without exact transition probabilities, it may be overly conservative when \( \epsilon^i \) significantly overestimates the true uncertainty \( \epsilon \). To reduce conservativeness, we introduce a second lower bound that is refined in real time as robots interact with the environment. Such a data-driven method allows the lower bounds of \( {P}^{\epsilon}_{i,k} \) to be adaptively refined for greater accuracy.
 
 For robot \( r_i \) with an initial state \( p_i \), the outcome of the policy \( \pi^{\epsilon}_{i, k} \) (either satisfying or violating the TWTL constraint) can be modeled as a Bernoulli trial with the probability of success equal to \( P^{\epsilon}_{i, k}(p_i) \). To estimate \( P^{\epsilon}_{i, k}(p_i) \), we use the Wilson score interval \cite{wilson1927probable}, which computes a confidence interval \([P^{\epsilon^-}_{i, k}(p_i), P^{\epsilon^+}_{i, k}(p_i)]\) that bounds \( P^{\epsilon}_{i, k}(p_i) \) with a specified confidence level, where
 \begin{equation} \footnotesize\label{lower} P^{\epsilon^-}_{i, k}(p_i) =  \frac{n^S_{i,k}(p_i)+\frac{1}{2}z^2}{n_{i,k}(p_i)+z^2} - \frac{z}{n_{i,k}(p_i)+z^2}\sqrt{\frac{n^S_{i,k}(p_i) n^F_{i,k}(p_i)}{n_{i,k}(p_i)}+\frac{z^2}{4}}.
\end{equation}
Here, \( n_{i,k}(p_i) \) represents the number of episodes in which robot \( r_i \) started at \( p_i \) and adopted \( \pi^\epsilon_{i, k} \). The number of episodes that satisfied or violated the constraint are denoted by \( n^S_{i,k}(p_i) \) and \( n^F_{i,k}(p_i) \), respectively. Thus, $n_{i,k}(p_i)=n^F_{i,k}(p_i)+ n^S_{i,k}(p_i)$. The value of $z$ is chosen based on the desired confidence level (e.g., 99\% confidence level corresponds to $z$ = 2.58). By selecting a sufficiently high value of $z$, we ensure that $P^{\epsilon}_{i, k}(p_i) \geq P^{\epsilon^-}_{i, k}(p_i)$ with high confidence. 

In the initial episodes, we rely on the static lower bound \( \underline{P}^{\epsilon}_{i,k} \), as the confidence lower bound \( P^{\epsilon^-}_{i,k} \) is initially far from the true probability \( P^{\epsilon}_{i,k} \) due to the limited data available. After each episode, we update \( P^{\epsilon^-}_{i,k}(p_i) \) for robot \( r_i \) and the task \( t_k \) it executed. As robots collect more data, \( P^{\epsilon^-}_{i,k} \) becomes a more accurate estimate of \( P^{\epsilon}_{i,k} \) and can gradually replaces the static lower bound. We denote this set of confidence lower bounds as \( \{P^{\epsilon^-}_{i,k}\} \) and refer to it as a set of \textbf{\textit{confidence lower bounds}}.

Finally, we discuss how to obtain \( V_{i, k} \) for problem \eqref{eqn:optimization}. As defined earlier, \( V_{i, k} \) represents the expected reward for robot \( r_i \) under its current policy, corresponding to the state value function of the MDP under \( \pi^{\epsilon}_{i, k} \). We estimate \( V_{i, k} \) using the TD(0) method \cite{sutton2018reinforcement}. 

\subsubsection{Single-Agent Policy for Reward Maximization ($k = K + 1$)}

In this case, the robot is not assigned to any TWTL tasks and exclusively focuses on exploring the environment to maximize its reward function. In this work, we use the tabular Q-learning algorithm \cite{watkins1992q} to learn an optimal policy that maximizes the expected discounted reward. However, any RL algorithm can be used for this purpose. The Q-function is updated as follows:
\begin{equation}\footnotesize
    Q_{i, k}(s^i_t, a^i_t) \leftarrow Q_{i, k}(s^i_t, a^i_t) + \alpha \left( r^i_{t+1} + \gamma \max_{a'} Q_{i, k}(s^i_{t+1}, a') - Q_{i, k}(s^i_t, a^i_t) \right),
\end{equation}
where \( s^i_t \), \( a^i_t \), \( r^i_{t+1} \), and \( \gamma \) denote the state, action, reward, and discount factor, respectively. At each state, the value function \( V_{i,K+1} \) is the maximum \( Q \)-value over all available actions at that state.


\begin{algorithm}[htb!]
\footnotesize
\SetKwInOut{Input}{Input}
\SetKwInOut{Output}{Output}
\Input{\justifying{A set of tasks $\{K+1\}$; a set of policies $\pi_{i, k}^{\epsilon}$;\, MDP $\mathcal{M}^i$ for robot $r_i$;  episode length $T$; initial MDP state $s_{0}$ }}
\caption{Single-agent Task Execution}\label{alg:single-agent}
\DontPrintSemicolon
\SetKwFunction{execute}{execute}
\SetKwFunction{update}{update}
\SetKwFunction{stats}{stats}
    \SetKwProg{Pn}{Function}{:}{\KwRet}
    \textbf{Initialization:} \\
    $\{\pi^{\epsilon}_{i, k}\}, \{\mathcal{P}_{i, k}\} \leftarrow$ Alg.~\ref{alg:offline}\\
    \For{ $\mathcal{P}_{i, k} \in$ $\{\mathcal{P}_{i, k}\}_{k=1, 2, \dots, K }$}{
         \For{$p \in$ all initial states of $\mathcal{P}_{i, k}$}{
         $n_{i,k}(p) \leftarrow 0, n^S_{i,k}(p) \leftarrow 0, n^F_{i,k}(p) \leftarrow 0$ \\ $\underline{P}^{\epsilon}_{i,k} \leftarrow \text{static lower bound}$
        }  
         $V_{i,k}(p) \leftarrow 0$ \textbf{for} $p \in$ all states of $\mathcal{P}_{i, k}$
        
     }
    $Q(s, a) \leftarrow 0$ \textbf{for} $(s,a) \in$ all state-action pairs of $\mathcal{M}^i$  \\
    $\pi(s) \leftarrow \argmax_{a} Q(s, a)$ \\
    \Pn{\execute{$\{P_{i,k}\}_{k=1, 2, \dots, K + 1}$}}{ 
    $k \leftarrow$ Sampled from $\{P_{i,k}\}_{k=1, 2, \dots, K + 1}$\\
    $p \leftarrow \bar{p}\in$ initial states of $\mathcal{P}_{i, k}$ s.t. mdp\_state($\bar{p}$) = $s_0$ \\
    $p_0 \leftarrow p$,  $s \leftarrow s_0$\\
    \For{$t = 1:T$}{
        \uIf{not constraint\_satisfied \textbf{and} $k \leq K$ }{
             $a \leftarrow$ Select action from policy $\pi^{\epsilon}_{i, k}(p)$\\
            Take action $a$, observe $p'$ (next state) and $r$ (reward)\\ $V_{i,k}(p) \leftarrow V_{i,k}(p) + \alpha \left( r + \gamma V_{i,k}(p') - V_{i,k}(p) \right)$ \\
            $p \leftarrow p', s \leftarrow mdp\_state(p')$
        }
        \ElseIf{constraint\_satisfied  \textbf{or} $k=K+1$}{
            $a \leftarrow$ Select action from $\pi(s)$ via $\epsilon-$greedy \\
            Take action $a$, observe $s'$ (next state) and $r$ (reward)\\
            $Q(s, a) \leftarrow Q(s, a) + \alpha \left( r + \gamma \max_{a'} Q(s', a') - Q(s, a) \right)$ \\
            Update policy $\pi(s)$ \\
            $s \leftarrow s'$
            
        }
        
    }
    $\FuncSty{update(}p_{0}, n_{i,k}(p_{0}),n^S_{i,k}(p_{0}),n^F_{i,k}(p_{0}),\textit{constraint\_satisfied}\FuncSty{)}$ \\
    $s_0 \leftarrow s$\\
    }       
    \Pn{\update{$p_{0}$, $n_{i,k}(p_{0})$, $n^S_{i,k}(p_{0})$, $n^F_{i,k}(p_{0})$, $satisfied$}}{ 
    $n_{i,k}(p_{0}) \leftarrow n_{i,k}(p_{0}) + 1$\\
    \lIf{satisfied}{
        $n^S_{i,k}(p_{0}) \leftarrow n^S_{i,k}(p_{0}) + 1$
    }
    \lElse{
        $n^F_{i,k}(p_{0}) \leftarrow n^F_{i,k}(p_{0}) + 1$
    }
    $P^{\epsilon^-}_{i, k}(p_{0}) \leftarrow$ equation (\ref{lower})
    }
    \Pn{\stats{}}{
     \For{$k = 1:K $}{
        $p \leftarrow \bar{p}\in$ initial states of $\mathcal{P}_{i, k}$ s.t. mdp\_state($\bar{p}$) = $s_{0}$ \\
        $V_{i, k} \leftarrow V_{i, k}(p)$, $\underline{P}^{\epsilon}_{i,k} \leftarrow \underline{P}^{\epsilon}_{i,k}(p)$ \\
        $\underline{P}^{\epsilon^-}_{i,k} \leftarrow P^{\epsilon^-}_{i,k}$ \textbf{if} $n_{i,k}(p) \geq N$ \textbf{else} $\underline{P}^{\epsilon^-}_{i,k} \leftarrow \underline{P}^{\epsilon}_{i,k}$ 
     }
     $V_{i, K+1} \leftarrow max_a Q(s_0, a)$, $\underline{P}^{\epsilon}_{i,K+1} \leftarrow 0$, $\underline{P}^{\epsilon^-}_{i,K+1} \leftarrow 0$ \\
     \textbf{return} $\{V_{i, k}\}_{i=1, \dots, K+1}$, $\{\underline{P}^{\epsilon}_{i,k}\}_{i=1, \dots, K+1}$, $\{\underline{P}^{\epsilon^-}_{i,k}\}_{i=1, \dots, K+1}$
    }  
\end{algorithm} 

\subsubsection{Combined Single-Agent Policy}
We summarize the overall single-agent policy in Alg.~\ref{alg:single-agent}. The algorithm begins by constructing the product MDPs and computing $\pi^{\epsilon}$ policies (line 2). Then it initializes the variables for the Wilson score lower bound (line 5), computes the static lower bound (line 6), initializes a value function table (line 7) and a Q-table (line 8), and its corresponding policy (line 9). Upon receiving task allocation $\{P_{i,k}\}$ from the high-level coordinator, a task is sampled from this probability distribution (line 11), and the robot's initial product MDP state is set accordingly (lines 12-13). Depending on the selected task, the robot either (i) select actions from policy $\pi^{\epsilon}_{i, k}$ and update the value function $V^{\epsilon}_{i, k}$ (lines 15-19) if executing a TWTL task ($k \leq K$) or (ii) learn a policy to maximize rewards via Q-learning (lines 20-25), if in an exploration phase ($k = K + 1$) or if the TWTL constraint is met before the episode ends. After executing the task, the robot updates the success or failure counts and the Wilson score lower bound (line 26). 

Function \stats{} is used to provide the expected rewards $\{V_{i,k}\}$, lower bounds $\{\underline{P}^{\epsilon}_{i,k}\}$ and $\{\underline{P}^{\epsilon^-}_{i,k}\}$ to the high-level coordinator (lines 33-39). The set \( \{\underline{P}^{\epsilon^-}_{i,k}\} \) initially consists of static lower bounds \( \underline{P}^{\epsilon}_{i,k} \), which are gradually replaced by the confidence lower bounds \( P^{\epsilon^-}_{i,k} \) as more data is collected (line 37). In line 37, $N$ is a user-defined data count threshold for switching from static to confidence lower bounds. We refer to \( \{\underline{P}^{\epsilon^-}_{i,k}\} \) as a set of \textbf{\textit{adaptive lower bounds}}.

\begin{remark}
   The single-agent policies provide high-level movement guidance for task execution (e.g., determining which grid cells to visit next, as shown in Figure~\ref{fig:example}). These policies do not explicitly handle low-level motion planning or collision avoidance. During actual execution, real-time collision avoidance algorithms such as \cite{van2008reciprocal, panagou2014motion, van2011reciprocal} and motion planning algorithms \cite{rosmann2012trajectory,  karaman2011sampling} can be integrated to ensure the safe and efficient navigation of robots.
\end{remark}

\subsection{Bi-level Task Allocation and Policy Learning}

\begin{assumption}\label{assumption2}  
We assume that, after substituting \( P^{\epsilon}_{i,k} \) with the static lower bound \( \underline{P}^{\epsilon}_{i,k} \) in the optimization problem \eqref{eqn:optimization}, a feasible solution \( \{P_{i,k}\} \) exists for the resulting optimization problem.
 \end{assumption}
 
This assumption sets a necessary condition for the proposed algorithm. The static lower bound \( \underline{P}^{\epsilon}_{i,k} \) reflects the maximum achievable probability of constraint satisfaction given the robot's limited knowledge of transition dynamics. If a feasible task allocation \( \{P_{i,k}\} \) does not exist, task satisfaction cannot be guaranteed at the desired probability thresholds, implying a need to either relax the TWTL task requirements or increase the number of robots.

We summarize the bi-level multi-robot coordination in Alg.~\ref{alg:bi-level}. Before each episode, the high-level coordinator collects task satisfaction probabilities and expected rewards from the robots (lines 2-3). The coordinator first attempts to solve the task assignment problem using the adaptive lower bounds (line 4). If no feasible solution is found, it falls back on the static lower bounds \( \underline{P}^{\epsilon}_{i,k} \), as Assumption 2 ensures that a feasible solution exists (lines 5-6). Once the task assignment is determined, each robot independently executes its policy (Alg.~\ref{alg:single-agent}) to complete the assigned task (lines 7-8).

\begin{algorithm}[htb!]
\footnotesize
\SetKwInOut{Input}{Input}
\Input{\justifying{A set of robots $\{N\}$; number of episodes $N_{episode}$}}
\caption{Bi-level Multi-robot Coordination}\label{alg:bi-level}
\DontPrintSemicolon
 \For{$j = 1:N_{episode}$}{
\For{$r_i \in \{N\}$}{
$\{V_{i, k}\}$, $\{\underline{P}^{\epsilon}_{i,k}\}$, $\{\underline{P}^{\epsilon^-}_{i,k}\} \leftarrow r_i.\stats{}$ (Alg.~\ref{alg:single-agent}) 
}
 
\textbf{try:} $\{P_{i,k}\} \leftarrow$ substitute $P^{\epsilon}_{i, k}$ with $\underline{P}^{\epsilon^-}_{i,k}$ in Problem~(\ref{eqn:optimization}) and solve 
    
 \If{no feasible solution $\{P_{i,k}\}$ found}{
            $\{P_{i,k}\} \leftarrow$ substitute $P^{\epsilon}_{i, k}$ with $\underline{P}^{\epsilon}_{i,k}$ in Problem~(\ref{eqn:optimization}) and solve 
        }
\For{$r_i \in \{N\}$}{
$ r_i.\execute{$\{P_{i,k}\}_{k=1, 2, ..., K + 1}$}$ (Alg.~\ref{alg:single-agent}) 
}

 }
\end{algorithm} 

\begin{remark}
In cases where some robots fail, re-planning can be achieved by solving a modified version of the original problem (Equation 3). The reallocation process focuses on the affected TWTL tasks and redistributes them among available robots that do not commit to any TWTL tasks. This minimizes disruptions to ongoing tasks and improves robustness against failure.

\end{remark}
\begin{theorem}\label{theorem}
Given a set of robots $\{N\}$, a set of TWTL tasks $\{K+1\}$, and a set of MDPs $\{\mathcal{M}\}$ subject to Assumption 1. Suppose Assumption 2 holds. Then, in every episode where Alg.~\ref{alg:bi-level} finds a feasible solution $\{P_{i,k}\}$, the probability of satisfying $\phi_k$ for each task $t_k$ is guaranteed to be at least $P_k$ with high confidence. \end{theorem}
\begin{proof}
    Theorem \ref{theorem} follows directly from Proposition \ref{proposition}. Full proof is available in the appendix.
\end{proof}


\section{Simulation}
To validate the proposed algorithm, we conduct simulations based on the pickup and delivery example illustrated in Fig.~\ref{fig:example}, incorporating time window requirements. The environment depicted in Fig.~\ref{fig:env} is used across all experiments. In Fig.~\ref{fig:env}, the colored cells represent stations (\(\text{S}_1\), \(\text{S}_2\)) where robots can idle, warehouses (\(\text{W}_1\), \(\text{W}_2\)) for storing resources, resource processing units (\(\text{P}_1\), \(\text{P}_2\)), and an operation site (O) where processed resources are delivered. The black zones indicate restricted areas that robots must avoid, while the blue zone represents water, which ground robots are prohibited from entering. The gray zone represents an area of interest requiring aerial monitoring, which is unknown to the robots. The arrows show two one-way bridges over the water. 
 

 
\begin{figure}[htbp]
    \centering
    \begin{subfigure}[t]{0.33\linewidth}
        \includegraphics[width=\linewidth]{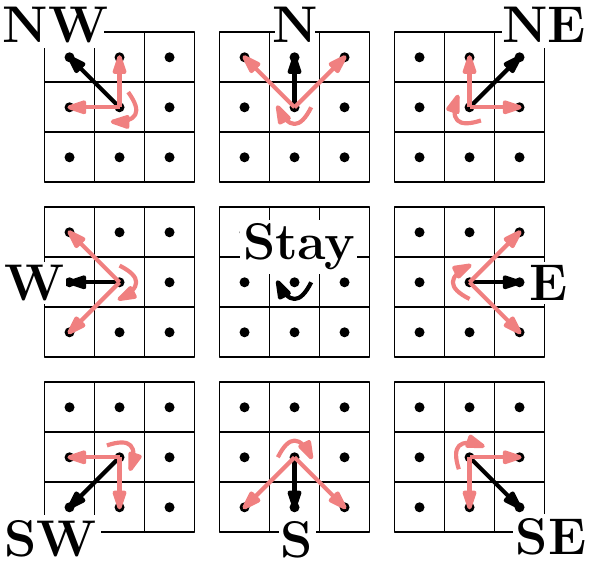}
        \caption{}
        \label{fig:transitions}
    \end{subfigure}
    \hspace{0.4cm}
    \begin{subfigure}[t]{0.6\linewidth}
    \includegraphics[width=\linewidth]{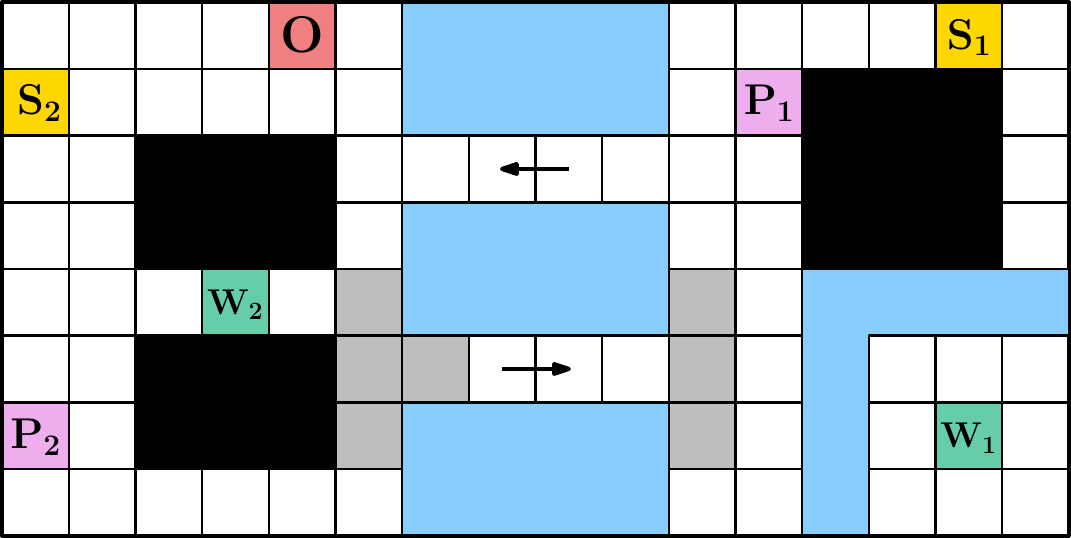}
        \caption{}
        \label{fig:env}
    \end{subfigure}
    \caption{(a) Transitions (intended - black, unintended - red) under each action. (b) An environment with different types of locations denoted by different colors.}
    
\label{fig:sim}
\end{figure} 
 
\begin{table}[htbp]
    \centering
    \footnotesize
    \caption{Robot Types, Rewards, and Uncertainties}
    \begin{tabular}{|c|c|c|c|c|}
        \hline
        \textbf{Robot Type} & \textbf{Robot Index} & \textbf{Reward} & \boldmath$\epsilon$ & \boldmath$\epsilon^i$ \\
        \hline
        \multirow{2}{*}{Drone} & 1, 2 & 5 at grey cells, 0 at other cells & \multirow{2}{*}{0.1} & \multirow{2}{*}{0.2} \\ \cline{2-3}
            & 3, 4 & 1 at grey cells, 0 at other cells &  &  \\
        \hline \hline
        \multirow{2}{*}{Mobile Robot} & 5, 6 & \multirow{2}{*}{1 at $S_2$, 0 at other cells} & \multirow{2}{*}{0.15} & 0.3 \\ \cline{2-2}\cline{5-5}
         & 7, 8 &  &  & 0.25 \\
        \hline
    \end{tabular}
\label{table2}
\end{table}

We consider robots with transition models described in Example \ref{example:transition}. The robot specifications are summarized in Table~\ref{table2}. Drones are incentivized to identify and monitor traffic congestions (grey cells with initially unknown locations to drones),  where drones 1 and 2 get higher rewards due to their better sensing capability. Mobile robots are encouraged to return to Station $S_2$ as part of their auxiliary objectives. All robots have a conservative estimate $\epsilon^i$ of their actual unknown transition uncertainty $\epsilon$. 

Table~\ref{table1} presents the task specifications. Tasks $\phi_1$ and $\phi_2$ require robots to pick up resources from warehouse $W_2$, process them at $P_2$ or $P_1$, respectively, and deliver them to $O$—all within 15 time steps, ensuring a 90\% success rate. Tasks $\phi_3$ and $\phi_4$ involve transporting resources from a different warehouse $W_1$, delivering them to $P_1$, and to $O$, with different time window requirements. The lower desired probabilities for $\phi_3$ and $\phi_4$ provide robots with more flexibility to pursue their auxiliary task of environmental monitoring while still ensuring that at least one of these tasks is successfully completed with high probability (approximately 90\%).

\begin{table}[h!]
\centering
\small
\caption{Task Specifications as TWTL Formulas}
\begin{tabular}{|c|c|c|}
\hline
\textbf{Task} & \textbf{TWTL Formula} & \makecell{\textbf{Desired} \\ \textbf{Probability}}\\ 
\hline
$\phi_1$ & $[H^1W_2]^{[0,15]} \cdot [H^1P_2]^{[0,15]} \cdot [H^1O]^{[0,15]}$ & 0.9\\
\hline
$\phi_2$ & $[H^1W_2]^{[0,15]} \cdot [H^1P_1]^{[0,15]} \cdot [H^1O]^{[0,15]}$ & 0.9\\
\hline
$\phi_3, \phi_4$ & $[H^1W_1]^{[0,20]} \cdot [H^1P_1]^{[0,10]} \cdot [H^1O]^{[0,15]}$ & 0.7\\
\hline
\end{tabular}
\label{table1}
\end{table}

We utilize the Robot Operating System (ROS 2) \cite{doi:10.1126/scirobotics.abm6074}, where each robot is implemented as an independent ROS node. Robots communicate with the coordinator using ROS communication mechanisms, enabling realistic simulation of a distributed multi-robot system. Simulations are conducted over $N_{episode}$ $=2000$ episodes per iteration. The results presented are averaged over 20 iterations. Details of the parameters used are provided in Appendix~\ref{C}.

\textbf{Case 1.} We evaluate the proposed algorithm under two conditions: using static lower bounds (by skipping line 4 in Alg.~\ref{alg:bi-level}) and adaptive lower bounds (as described in Alg.~\ref{alg:bi-level}). 
Our results show that solving the task allocation problem with static lower bounds tends to assign more robots to TWTL tasks ($\phi_1$ - $\phi_4$), compared to the adaptive lower bounds. This finding aligns with our expectations, as the static lower bound computes task satisfaction probability based on $\epsilon^i$, which overestimates the actual uncertainty $\epsilon$. In contrast, the adaptive lower bounds refine the estimate of task satisfaction probability as the robot interacts with the environment, resulting in a less conservative lower bound. With the adaptive lower bounds, robots progressively gain confidence in TWTL satisfaction, allowing the high-level coordinator to adaptively improve the task allocation by assigning fewer robots to TWTL tasks. 

Fig.~\ref{fig:sat_rate} shows that task satisfaction rates consistently exceed the desired probability over 2000 episodes. This confirms that the proposed framework successfully ensures probabilistic satisfaction of the TWTL tasks, regardless of the lower bound selection. Nevertheless, static lower bounds result in a conservatively higher satisfaction rate. Fig.~\ref{fig:sat_rate_ada} compares the TWTL satisfaction rate in the first 100 episodes versus 2000 episodes when using adaptive lower bounds. We observe that as the adaptive lower bounds refine over time, satisfaction rates converge toward the desired probabilities by reducing assignments to TWTL tasks. As a result, using the adaptive lower bounds yields higher rewards, as shown in Fig.~\ref{fig:reward}. 

\begin{figure}[htbp]
    \begin{subfigure}[t]{0.49\linewidth}
        \includegraphics[width=\linewidth]{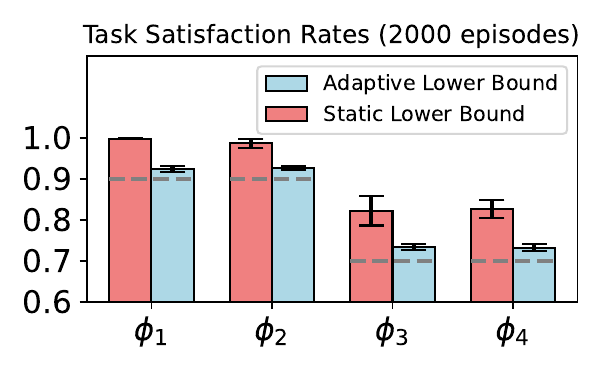}
        \caption{Static vs. Adaptive}
        \label{fig:sat_rate}
    \end{subfigure}
    \begin{subfigure}[t]{0.49\linewidth}
        \includegraphics[width=\linewidth]{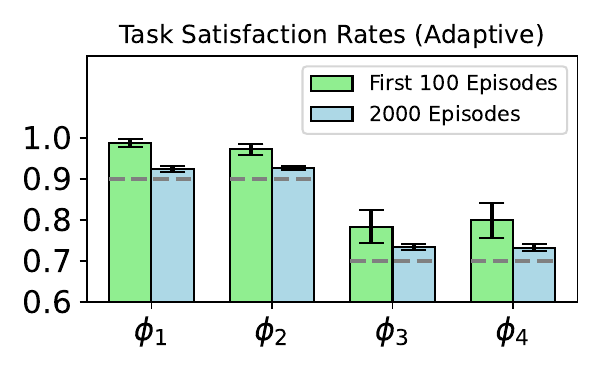}
        \caption{Initial vs. Full Episodes}
        \label{fig:sat_rate_ada}
    \end{subfigure}
    \caption{TWTL Task satisfaction rate. (a): Satisfaction rate over 2000 episodes using static and adaptive lower bounds. (b) Satisfaction rate using adaptive lower bounds, over the first 100 episodes vs 2000 episodes. The dashed line represents the desired probability of each TWTL task. The bars show the mean values over the 20 iterations, with error bars depicting the standard deviation.
    }
\label{fig:sim_result}
\end{figure}

Our framework, which incorporates auxiliary reward functions alongside the TWTL constraint, enables adaptive and self-improving task allocation. In the early stages of each iteration, the system predominantly assigns TWTL tasks to the drones (Robots 1-4) as the mobile robots (Robots 5-8) initially have high transition uncertainty estimates ($\epsilon^i$), resulting in lower task satisfaction probability estimates.
As the simulation progresses, the robots continuously update their task satisfaction probabilities. Meanwhile, drones gradually discover high-reward areas, prompting the coordinator to adjust the task assignments. By the end of the simulation, the system converges to a reasonable division of labor, with mobile robots (Robots 5-8) and low-reward drones (Robots 3-4) assigned to TWTL tasks, while drones with high-quality sensors (Robots 1-2) remain unassigned to focus on monitoring the high-interest area. This adaptive behavior demonstrates the system’s ability to dynamically reallocate resources, balancing reward maximization with guaranteed satisfaction of TWTL constraints.

\textbf{Case 2.} In this case study, we evaluate the computation time of the proposed framework as the number of robots increases. The experimental setup is identical to Case 1, where tasks listed in Table~\ref{table1} must be satisfied with their corresponding probability thresholds. We simulate systems with 8, 16, 24, 32, 40, and 48 robots, duplicating the robots defined in Table~\ref{table2} by factors from 1 to 6. In this setup, we utilize the adaptive lower bound.

Fig.~\ref{fig:computation_time_vs_robots} presents the average computation time per episode for solving the task allocation problem \eqref{eqn:optimization}, averaged over 2000 episodes. The results show that, even with a large number of robots, the proposed framework solves the task assignment problem in a relatively short time. Notably, for 48 robots, the average computation time remains under 0.4 seconds, demonstrating the framework’s capability for efficient real-time applications.

\begin{figure}[htbp]
    \centering
    \begin{subfigure}[t]{0.49\linewidth}
        \includegraphics[width=\linewidth]{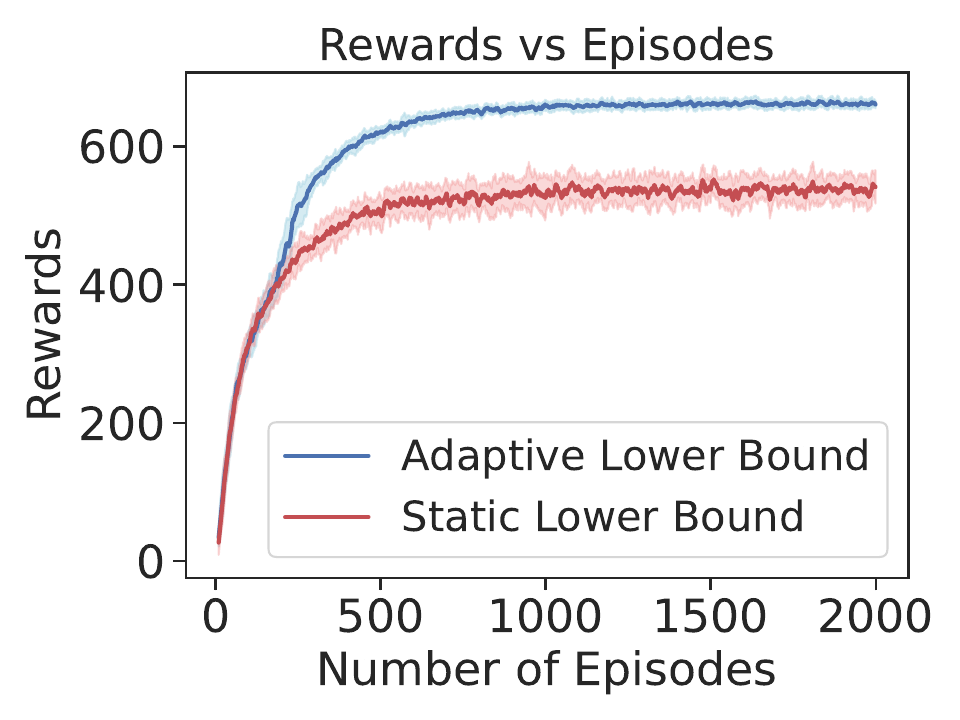}
        \caption{Static vs. Adaptive}
        \label{fig:reward}
    \end{subfigure}
    \begin{subfigure}[t]{0.49\linewidth}
        \includegraphics[width=\linewidth]{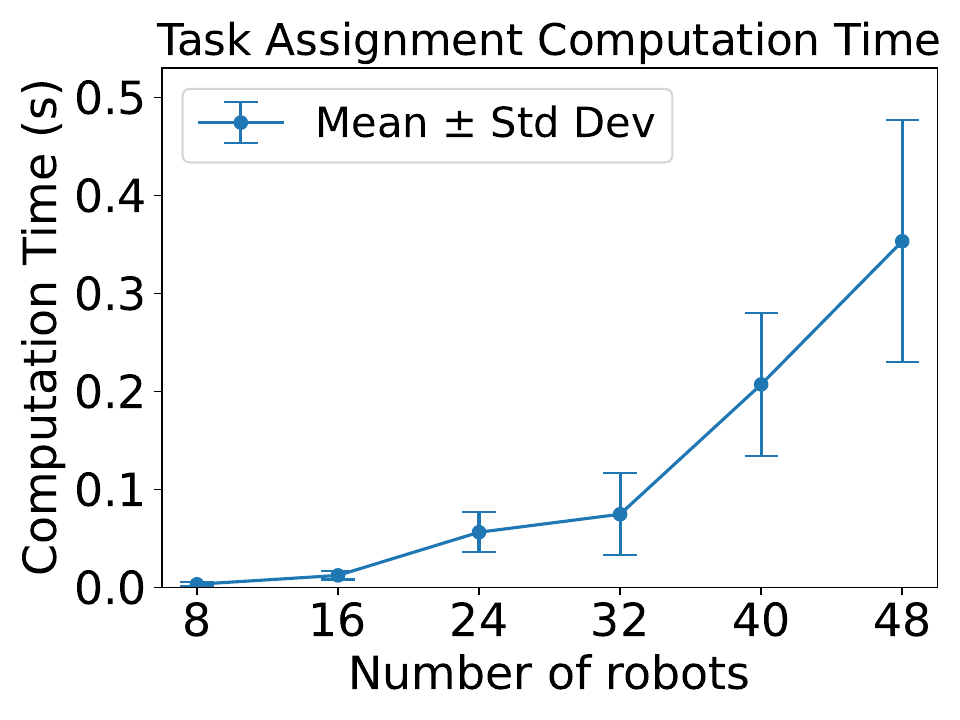}
        \caption{}
        \label{fig:computation_time_vs_robots}
    \end{subfigure}
    
    \caption{(a): Total rewards accumulated by all robots over the episodes. The reward represents the sum of individual rewards collected by each robot. (b): Computation time for solving task assignment (Alg.~\ref{alg:bi-level}, line 4-6) with different number of robots.}
\end{figure}




In Alg.~\ref{alg:offline}, each robot independently constructs the product MDPs and synthesizes the policies offline. In Alg.~\ref{alg:single-agent}, each robot executes its task independently, without the need for communication with others. Thus, computation time for these processes is unaffected by the number of robots in the system. While our evaluation focused on varying the number of robots, increasing the number of tasks similarly impacts computation time by adding nonlinear constraints to the optimization problem. Analysis of this effect can be found in Appendix~\ref{B2}. 


\section{Conclusion}
This paper presents a bi-level framework that integrates high-level task assignment with lower-level policy execution and learning to ensure robots satisfy Time-Window Temporal Logic (TWTL) tasks with guaranteed probability thresholds. We introduce adaptive lower bounds on task completion probabilities, enabling robots to iteratively refine their probability estimates for more informed task allocation decisions. By incorporating reward functions as an auxiliary objective, the system can iteratively improve task assignments to maximize expected rewards while maintaining probabilistic task satisfaction. The flexibility to impose different reward functions on robots enables the system to achieve multiple objectives simultaneously, allowing different reward functions to guide user-customized allocation plans without compromising the primary objective of satisfying TWTL tasks.

We provide theoretical analysis and conduct comprehensive simulations to validate the framework. The results highlight its ability to ensure constraint satisfaction at desired probability thresholds under uncertainty. Furthermore, simulation results also show that the task allocation problem can be efficiently solved for large numbers of robots in real time, demonstrating the framework's practical applicability in real-world scenarios.


\begin{acks}
We sincerely thank Dr. Kevin Leahy for his insightful discussions and valuable suggestions on this paper's problem formulation.
\end{acks}



\bibliographystyle{ACM-Reference-Format} 
\bibliography{arxiv}

\appendix
\section{Proof of Theorem \ref{theorem}}\label{A}
\begin{proof}
In Alg.~\ref{alg:bi-level}, a feasible task allocation $\{P_{i,k}\}$ for each episode is solved using either (i) the static lower bounds (lines 5-6); or (ii) the adaptive lower bounds (line 4). 

(i) By Proposition~\ref{proposition}, replacing $P^{\epsilon}_{i,k}$ with its static lower bounds $\underline{P}^{\epsilon}_{i,k}$ in the optimization problem ensures that the solution remains feasible for the original problem \eqref{eqn:optimization}. 

(ii) The adaptive lower bounds, defined in Alg.~\ref{alg:single-agent} (line 37), include either static lower bounds $\underline{P}^{\epsilon}_{i,k}$ or confidence lower bound $P^{\epsilon^-}_{i,k}$ which are lower bounds of $P^{\epsilon}_{i,k}$ (with high confidence though not 100\%). By Proposition~4.1, replacing \( P^{\epsilon}_{i,k} \) with these adaptive lower bounds ensures that the solution found in line 4 of Alg.~\ref{alg:bi-level} solves the original problem \eqref{eqn:optimization} with high confidence.

Considering both cases, we conclude that the feasible solution 
$\{P_{i,k}\}$ found by Alg.~\ref{alg:bi-level} satisfies the original problem \eqref{eqn:optimization} with high confidence. 
Specifically, this ensures that constraint \eqref{eqn:constraint1} holds, meaning that the probability of satisfying each TWTL task 
$t_k$ meets or exceeds its required threshold $P_k$.
\end{proof}

\textit{Notes: Theorem~4.6 does not guarantee that the probabilities of satisfying TWTL tasks are always greater than or equal to their desired thresholds; instead, it ensures this probabilistic satisfaction with high confidence. This is because $P^{\epsilon}_{i, k}$ was estimated using the Wilson score method, which means that \(P^{\epsilon}_{i,k} \geq P^{\epsilon^-}_{i, k}\) holds true with a high confidence level depending on the chosen parameter $z$ in (9). Additionally, though Assumption~\ref{assumption2} ensures the existence of a feasible solution \(\{P_{i,k}\}\) in line 6 of Alg.~\ref{alg:bi-level},  finding such a feasible solution is not always guaranteed due to the non-convexity of \eqref{eqn:constraint1}. Therefore, we aim to provide a probabilistic satisfaction guarantee whenever a feasible solution is found.}

\section{Extended Simulation Results}\label{B1}

\subsection{Extended Results for Case 1}
We extend our discussion of Case 1 with a more detailed analysis. We evaluate the task allocation strategy by analyzing the percentage of episodes in which each robot is assigned to TWTL tasks versus left unassigned to focus on reward maximization, as shown in Table \ref{table:robot task assignment}. The results show that the system consistently assigns TWTL tasks to Robots 3-8, while Robots 1 and 2 primarily focus on maximizing their auxiliary rewards.

This division of labor emerges naturally from our framework's design, which prioritizes satisfying TWTL tasks with the desired probability threshold (primary objective) while simultaneously optimizing the robots' auxiliary rewards (secondary objective). Specifically, robots with higher reward potential, such as Drones 1 and 2, are more likely to remain unassigned to TWTL tasks, allowing them to maximize their rewards by monitoring interested areas. Conversely, Drones 3 and 4, which exhibit lower reward potential, are more frequently assigned to TWTL tasks to ensure task satisfaction.

Furthermore, robots with higher task completion uncertainty due to high initial transition uncertainty estimates, such as Robots 5 and 6, are more often assigned to easier tasks (e.g., $\phi_1$, which requires traveling a shorter distance). In contrast, the more reliable Robots 7 and 8 are consistently assigned to relatively more difficult TWTL tasks (e.g., $\phi_2$, which require longer travel distances) to maintain the required task satisfaction thresholds.

Figure~\ref{fig:number_of_robots} presents the average number of robots assigned to each task over 20 iterations. As shown, utilizing adaptive lower bounds significantly reduces the number of robots allocated to TWTL tasks while increasing the number of unassigned robots, which explains the higher overall rewards shown in Fig.~\ref{fig:reward}. Additionally, this approach results in a satisfaction rate that is less conservative and more closely aligned with the desired probability thresholds, as shown in Fig.\ref{fig:sim_result}.

\begin{table}[hbp]
    \centering 
    \small
    \caption{Task Assignment Distribution with Adaptive Lower Bound}
    \begin{tabular}{|c||c|c|c|c|}
        \hline
        \multirow{2}{*}{\textbf{Robot}} & \multirow{2}{*}{\textbf{Reward (\%)}} & \multicolumn{3}{c|}{\textbf{TWTL Tasks (\%)}} \\
        \cline{3-5}
        &  & \textbf{$\phi_1$} & \textbf{$\phi_2$} & \textbf{$\phi_3$ or $\phi_4$} \\
        \hline
        \textbf{Robot 1}   & 94.73 ± 2.29  & 0.23 ± 0.12  & 1.26 ± 0.69  & 3.78 ± 1.66 \\
        \hline
        \textbf{Robot 2}   & 96.34 ± 1.36  & 0.15 ± 0.08  & 1.03 ± 0.50  & 2.48 ± 0.90 \\
        \hline
        \textbf{Robot 3}   & 24.45 ± 1.59  & 0.32 ± 0.13  & 3.43 ± 0.93  & 71.80 ± 1.76 \\
        \hline
        \textbf{Robot 4}   & 25.56 ± 1.29  & 0.32 ± 0.13  & 4.25 ± 0.80  & 69.86 ± 1.17 \\
        \hline
        \textbf{Robot 5}   & 47.12 ± 33.50 & 52.48 ± 33.45 & 0.40 ± 0.25   & 0.00 ± 0.00 \\
        \hline
        \textbf{Robot 6}   & 59.16 ± 32.96 & 40.51 ± 32.89 & 0.33 ± 0.19   & 0.00 ± 0.00 \\
        \hline
        \textbf{Robot 7}   & 58.57 ± 34.00 & 2.41 ± 0.88  & 39.02 ± 34.41 & 0.00 ± 0.00 \\
        \hline
        \textbf{Robot 8}   & 44.21 ± 33.92 & 3.50 ± 1.26  & 52.30 ± 33.79 & 0.00 ± 0.00 \\
        \hline
    \end{tabular}
    \label{table:robot task assignment}
\end{table}

\begin{figure}[hbp] 
    \centering
    \begin{subfigure}[t]{0.49\linewidth}
        \includegraphics[width=\linewidth]{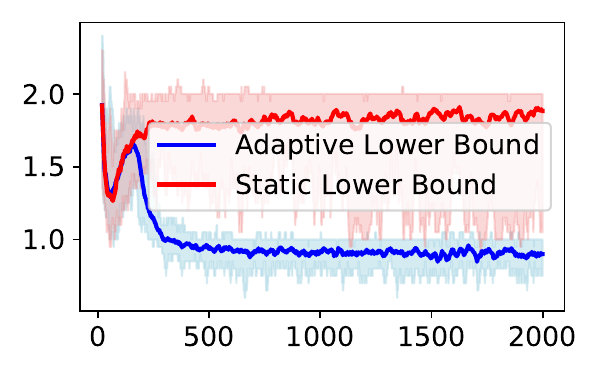}
        \caption{Task 1: $\phi_1$}
        \label{fig:task_1}
    \end{subfigure}
    \begin{subfigure}[t]{0.49\linewidth}
        \includegraphics[width=\linewidth]{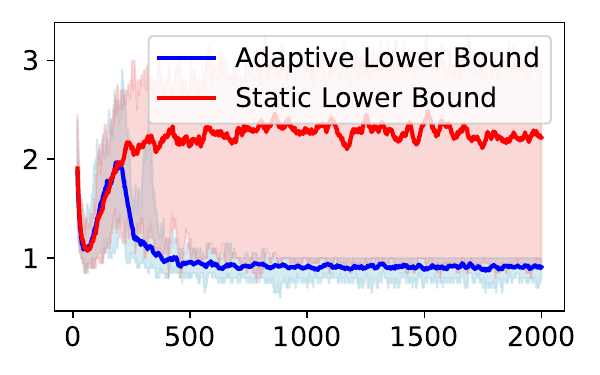}
        \caption{Task 2: $\phi_2$}
        \label{fig:task_2}
    \end{subfigure}
    \begin{subfigure}[t]{0.49\linewidth}
        \includegraphics[width=\linewidth]{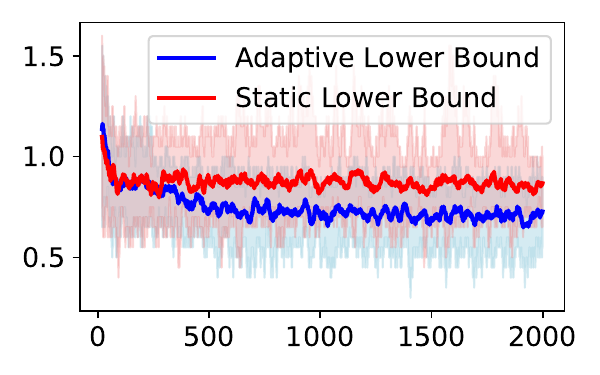}
        \caption{Task 3: $\phi_3$}
        \label{fig:task_3}
    \end{subfigure}
     \begin{subfigure}[t]{0.49\linewidth}
        \includegraphics[width=\linewidth]{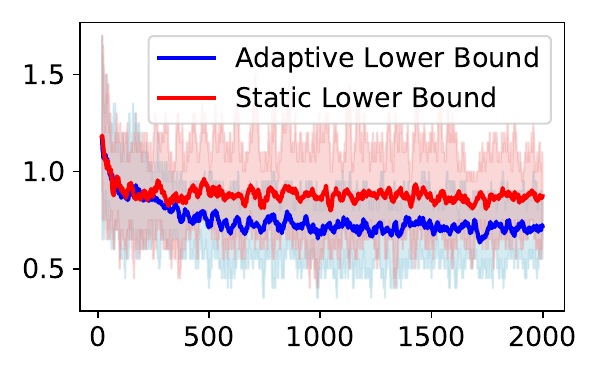}
        \caption{Task 4: $\phi_4$}
        \label{fig:task_4}
    \end{subfigure}
    \begin{subfigure}[t]{0.49\linewidth}
        \includegraphics[width=\linewidth]{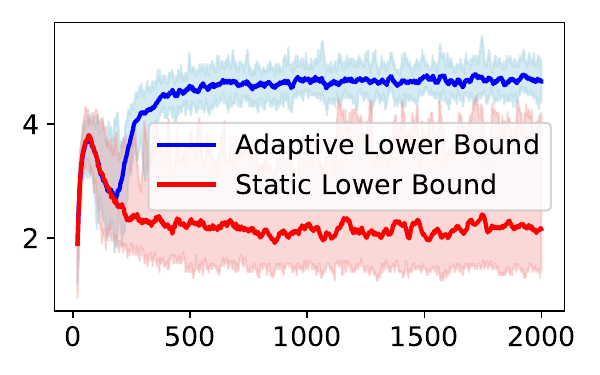}
        \caption{Reward Maximization}
        \label{fig:task_5}
    \end{subfigure}
    \caption{(a)-(e): Number of robots assigned to each task over episodes}
    \label{fig:number_of_robots}
    \end{figure}

\subsection{Extended Results for Case 2}\label{B2}

We report the computation time of the proposed framework when varying the number of tasks. We fix the number of robots to 20 by duplicating the 4 types of robots defined in Table~\ref{table2} by 5 times. We simulate the systems with 2, 4, 6, 8, and 10 tasks, selected from the tasks defined in Table~\ref{table1}. In this setup, we use the adaptive lower bound. Fig.~\ref{fig:computation_time_vs_tasks} presents the average computation time per episode for solving the task allocation problem \eqref{eqn:optimization}, averaged over 2000 episodes. Since increasing the number of tasks introduces more nonlinear constraints in \eqref{eqn:constraint1}, the observed superlinear growth in computation time aligns with our expectations. However, even for a moderate problem size—20 robots with 10 tasks—the framework remains computationally efficient, requiring an average of only 0.3 seconds per episode.

In Alg.~\ref{alg:offline}, each robot independently constructs the product MDPs and synthesizes policies offline. As the number of tasks increases, the number of product MDPs and policies that each robot must generate also grows. Fig.~\ref{fig:computation_time_init} reports the computation time for Alg.~\ref{alg:offline} under different numbers of tasks. In Alg.~\ref{alg:single-agent}, since each robot executes only one policy per episode, task execution remains unaffected by the number of tasks.

\section{Simulation Setup}\label{C}
We conduct simulations over a fixed number of episodes ($N_{episode}$ $=2000$) using diminishing $\epsilon$-greedy policy within the Q-learning algorithm (with $\epsilon_{init}=0.7$ and $\epsilon_{final}=0.0001$). The learning rate and discount factor are set to 0.1 and 0.95, respectively. We set the $z$ score to 2.58 to ensure the probabilistic constraint satisfaction with high confidence. The data count threshold for Alg. 3 line 37 is set to 40. We use the Scipy SLSQP solver to solve problem \eqref{eqn:optimization}.

\begin{figure}[htbp]
    \begin{subfigure}[t]{0.49\linewidth}
        \includegraphics[width=\linewidth]{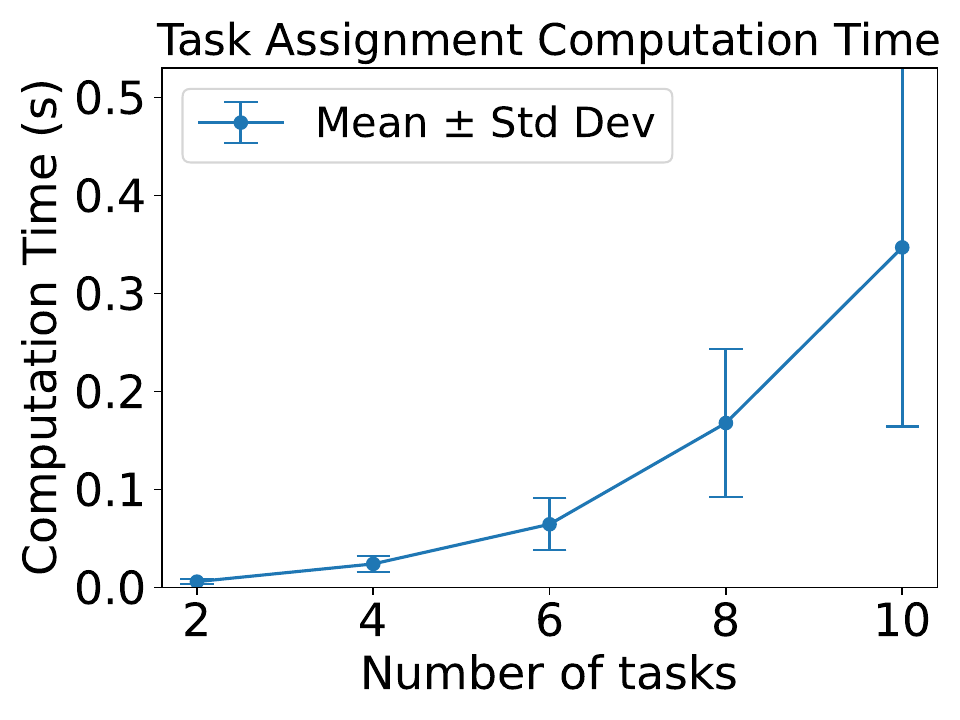}
        \caption{}
        \label{fig:computation_time_vs_tasks}
    \end{subfigure}
    \begin{subfigure}[t]{0.49\linewidth}
        \includegraphics[width=\linewidth]{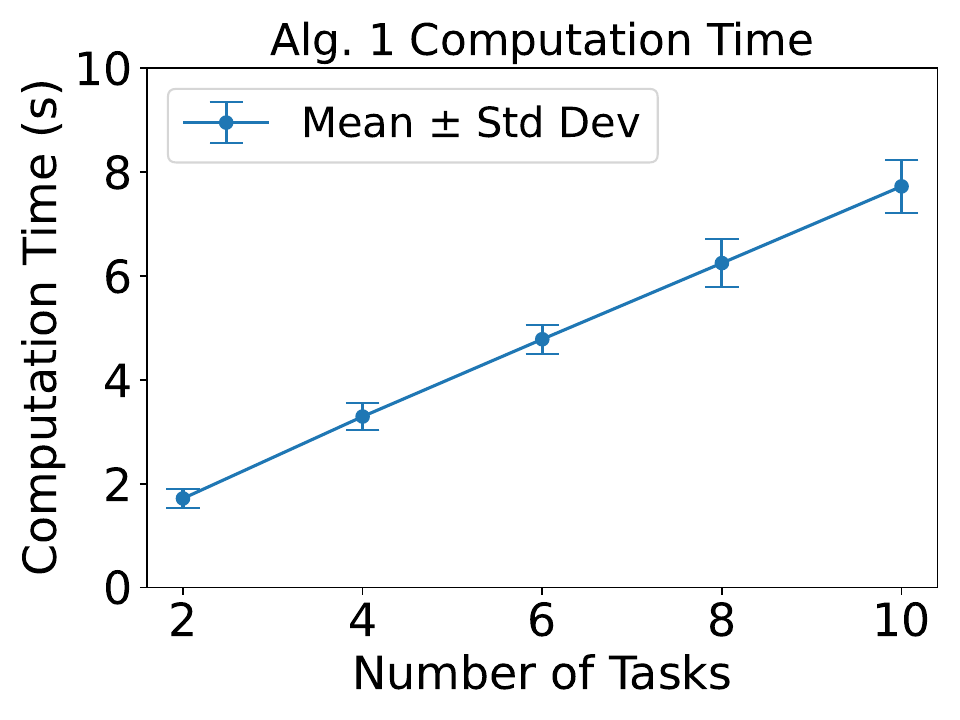}
        \caption{}
        \label{fig:computation_time_init}
    \end{subfigure}
    \caption{(a): Computation time for solving task assignment (Alg.~\ref{alg:bi-level}, line 4-6) with different number of tasks. (b) Computation time for Alg. 1 under different number of tasks. 
    }
\label{fig:time_vs_task}
\end{figure}


\end{document}

%% file: commands.tex
\newtheorem{assumption}{\bf{Assumption}}
\newtheorem{problem}{\bf{Problem}}
\newtheorem{remark}{\bf{Remark}}

\renewcommand{\thesubsubsection}{\Alph{subsubsection}.}